\documentclass[preprint,12pt]{elsarticle}
\journal{Artificial Intelligence}

\makeatletter
\def\ps@pprintTitle{%
   \let\@oddhead\@empty
   \let\@evenhead\@empty
   \let\@oddfoot\@empty
   \let\@evenfoot\@oddfoot
}
\makeatother

\usepackage{booktabs} 
\usepackage{amssymb}
\usepackage{amsmath}
\usepackage{wasysym}
\usepackage{epigraph}
\usepackage{graphics}
\usepackage{setspace}
\usepackage{fancybox} 
\newtheorem{definition}{Definition}
\newtheorem{theorem}{Theorem}
\newtheorem{lemma}{Lemma}
\newtheorem{claim}{Claim}

\newenvironment{proof}{\noindent{\sf Proof.}}{\hfill $\boxtimes\hspace{2mm}$\linebreak}
\renewcommand{\phi}{\varphi}
\renewcommand{\epsilon}{\varepsilon}
\newcommand{\cF}{\overline{F}}
\newenvironment{proof-of-claim}{\noindent{\sc Proof of Claim.}}{\hfill $\boxtimes\hspace{2mm}$\linebreak}

\renewcommand{\phi}{\varphi}
\renewcommand{\epsilon}{\varepsilon}

\newsavebox{\diamonddotsavebox}
\sbox{\diamonddotsavebox}{$\Diamond$\hspace{-1.8mm}\raisebox{0.3mm}{$\cdot$}\hspace{1mm}}

\begin{document}
\begin{frontmatter}
\title{Strategic Coalitions in Stochastic Games}

\author{Pavel Naumov}
\address{Claremont McKenna College, Claremont, California, USA}
\ead{pgn2@cornell.edu}
\author{Kevin Ros}
\address{Vassar College, Poughkeepsie, New York, USA}
\ead{kevinros@vassar.edu}






\begin{abstract}
The article introduces a notion of a stochastic game with failure states and proposes two logical systems with  modality ``coalition has a strategy to transition to a non-failure state with a given probability while achieving a given goal." The logical properties of this modality depend on whether the modal language allows the empty coalition. The main technical results are a completeness theorem for a logical system with the empty coalition, a strong completeness theorem for the logical system without the empty coalition, and an incompleteness theorem which shows that there is no strongly complete logical system in the language with the empty coalition. 
\end{abstract}

\end{frontmatter}


\section{Introduction}

In this article we study coalition power in stochastic games. An example of such a game is the road situation depicted in Figure~\ref{intro-example figure}. In this situation, self-driving car $a$ is trying to pass self-driving car $b$. Unexpectedly, a truck moving in the opposite direction appears on the road. For the sake of simplicity, we assume that cars $a$ and $b$ have only three strategies: slow-down ($-$), maintain the current speed ($0$), and accelerate ($+$). We also assume that the truck is too heavy to significantly change the speed before a possible collision. If cars $a$ and $b$ cooperate, there are two sensible things that they can do: (i) car $b$ can accelerate letting car $a$ to slow down and to return to the position behind car $b$; (ii) car $b$ can slow down letting car $a$ to accelerate and to pass before it reaches the truck. 

\begin{figure}[ht]
\begin{center}
\scalebox{0.95}{\includegraphics{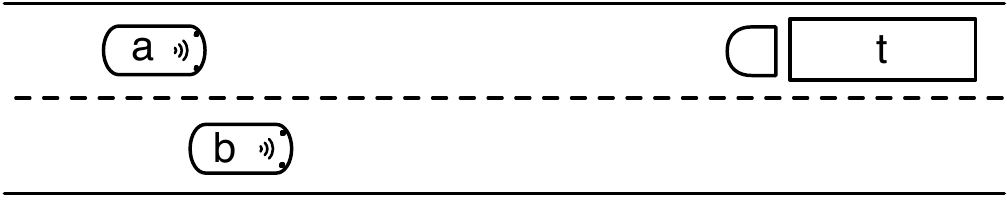}}
\caption{A Road Traffic Situation.}\label{intro-example figure}
\end{center}
\end{figure}

The diagram in Figure~\ref{intro-example game figure} describes probabilities of different outcomes of all possible combinations of actions of cars $a$ and $b$. This diagram has five states: state $p$ is the current (``passing") state of the system. States $ab$ and $ba$ represent outcomes in which car $a$ ends up, respectively, behind and ahead of car $b$. States $f_c$ and $f_t$ are ``failure'' states: in the first of them there is a collision between cars, in the second car $a$ collides with the truck.   The actual probabilities of possible outcomes for any give combination of actions are captured by the labeled directed edges. For example, the directed edge from state $p$ to state $ab$ labeled with $-+/1.0$ means that in the case (i) above, when car $a$ slows down ($-$) and car $b$ accelerates ($+$), the system safely transitions into state $ab$ with probability $1.0$. This means that coalition $\{a,b\}$ has a strategy that avoids collision with probability $1.0$. We write this as
$$
[a,b]_{1.0}(\mbox{``Collision is avoided''}).
$$
At the same time, directed edge from state $p$ to state $ba$ is labeled with $+-/0.9$. Hence, in the case (ii) above, the car $a$ will be able to pass car $b$ without collision with probability $0.9$:
$$
[a,b]_{0.9}(\mbox{``Pass without collision''}).
$$
The label $+-,0-/0.1$ on the directed edge from state $p$ to failure state $f_t$ denotes the fact that if car $a$ either accelerates ($+$) or maintains the same speed ($0$), while car $b$ slows down ($-$), then car $a$ will collide with the track with probability $0.1$. 

Note that car $a$ alone does not have a strategy to pass without collision with probability $0.9$. Indeed, if car $a$ decides to accelerate ($+$), then depending on if car $b$ slows down ($-$), maintains the current speed ($0$), or accelerates ($+$), the probability of passing without collision will be $0.9$, $0.6$, and $0.0$. Thus, although car $a$, of course, has a strategy to pass without collision with probability $0.0$:
$$
[a]_{0.0}(\mbox{``Pass without collision''}),
$$
it does not have a strategy to pass that would guarantee survival with any positive probability $\epsilon>0$: 
$$
\neg[a]_{\epsilon}(\mbox{``Pass without collision''}).
$$

\begin{figure}[ht]
\begin{center}
\scalebox{0.75}{\includegraphics{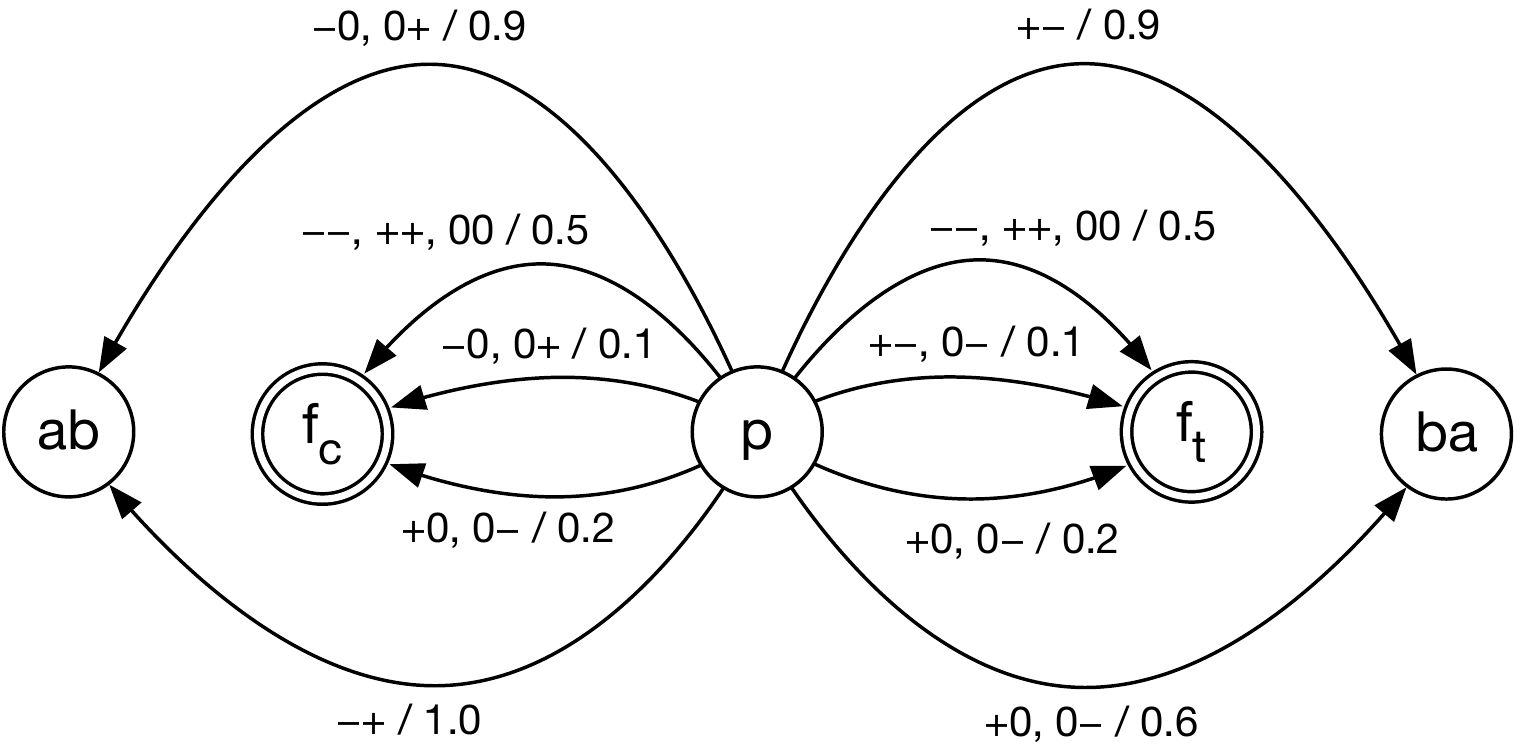}}
\caption{A Stochastic Game.}\label{intro-example game figure}
\end{center}
\end{figure}

In this article we study properties of modality $[C]_p\phi$ that stands for ``coalition $C$ has a strategy that achieves $\phi$ in all non-failure states and is guaranteed to avoid failure states with probability at least $p$''.
If $p=1.0$, then this is essentially coalition power modality introduced by Marc Pauly~\cite{p01illc,p02}. Pauly proved the completeness of the basic logic of coalition power. 
His approach has been widely studied in the literature~\cite{g01tark,vw05ai,b07ijcai,sgvw06aamas,av08aamas,abvs10jal,avw09ai,b14sr,ge18aamas,al18aamas,ga17tark,alnr11jlc}. 

Alur, Henzinger, and Kupferman introduced Alternating-Time Temporal Logic (ATL) that combines temporal and coalition modalities~\cite{ahk02}.
Goranko  and  van  Drimmelen~\cite{gd06tcs} gave a complete axiomatization of ATL. Decidability and model checking problems for ATL-like systems has also been widely studied~\cite{amrz16kr,bmmrv17lics,bmm17aamas}.
Chen and Lu added the probability of achieving a goal to ATL and developed model checking algorithms for the proposed system~\cite{cl07fskd}. 
Another version of ATL with probabilistic success was proposed by Bulling and Jamroga~\cite{bj09fi}. They considered modality $\langle\langle A\rangle\rangle^p_\omega\phi$ that stands for ``coalition $C$ can bring about $\phi$ with success level of at least $p$ when the opponents behave according to $\omega$'' and investigated its model checking properties. Unlike our approach in the current paper, neither of these works distinguish failure states from non-failure states. The probability of success in their systems is the probability of achieving $\phi$, not the probability of avoiding a failure state. Nov{\'a}k and Jamroga~\cite[Definition 2.5]{nj11ijcai} defined probability of  a ``successful execution" of an action as probability of an action to achieve its expected effect. They called failure any execution when an action does not achieve the expected (``annotated'') effect. This, in essence, is the probability studied in our article. Since we consider a multiagent setting, we find it more intuitive to talk about failure states rather than failure of individual actions or action profiles. Later in the paper~\cite{nj11ijcai}, however, Nov{\'a}k and Jamroga introduced modality $[\tau]_p\phi$ that refers to probability of an agent program $\tau$ to achieve goal $\phi$ and not a probability of individual actions to fail. Huang, Su, and Zhang combined perfect recall and coalition power to achieve a goal with a certain probability. They discussed model checking properties of their logical system~\cite{hsz12aaai}. Coalition power to achieve a goal with a certain probability is also used in PRISM-games, a model checker for stochastic multi-player games~\cite{cfkps13tools,kpw18ijsttt}. None of these works on  probabilistic extensions of ATL contain completeness results. 

Alternative approaches to  expressing the power to achieve a goal in a temporal setting are the STIT logic~\cite{bp90krdr,h01,h95jpl,hp17rsl,ow16sl} and Strategy Logic~\cite{chp10ic,mmpv14tocl,bmmrv17lics,ammr18ic}. Broersen, Herzig, and Troquard have shown that coalition logic can be embedded into a variation of STIT logic~\cite{bht07tark}. We are not aware of any probabilistic versions of either STIT or Strategy Logic.

\section{Outline of the Contribution}

In this article we axiomatize the properties of modality $[C]_p$ in stochastic game. It turns out that the axiomatization results depend significantly on whether the language allows the empty coalition or not. If the empty coalition is allowed, then one can use it to write formula $[\varnothing]_p\top$, which means that the system will unavoidably survive with probability $p$. Similarly, $[\varnothing]_p\phi$ means that the system will unavoidably survive with probability $p$ and statement $\phi$ will be true in the next state. Unavoidability cannot be expressed in the language without the empty coalition. In this article we introduce two different logical systems for modality $[C]_p$. The first of these systems, $\mathcal{L}$, allows coalition $C$ to be empty and the second, $\mathcal{L}^+$, does not. We describe the syntax and semantics of $\mathcal{L}$ and $\mathcal{L}^+$ in Section~\ref{syntax and semantics section}. We introduce the axioms and the inference rules for the logical systems in Section~\ref{logical systems section} and prove the soundness of the axioms in Section~\ref{soundness section}. The main technical contributions of this article are the completeness theorem for system $\mathcal{L}$ and the strong completeness theorem for system $\mathcal{L}^+$, which we prove in Section~\ref{completeness section}. Additionally, in Section~\ref{incompleteness section} we prove that no strongly sound logical system is strongly complete in the language with the empty coalition. In Section~\ref{decidability section} we discuss the decidability of the systems. Finally, we conclude the article in Section~\ref{conclusion section}.

\section{Syntax and Semantics}\label{syntax and semantics section}

In the current section we introduce the formal syntax and the formal semantics for logical systems $\mathcal{L}$ and $\mathcal{L}^+$. The language $\Phi$ of the first of these systems allows set $C$ to be empty and language $\Phi^+$ of the second does not. In both cases, we assume a fixed {\em finite} set of agents $A$ and a fixed set of propositional variables. Additionally, a {\em coalition} is any subset of $A$.  

\begin{definition}\label{Phi}
Let $\Phi$ be the minimal set of formulae such that
\begin{enumerate}
    \item $v\in \Phi$ for each propositional variable $v$,
    \item $\neg\phi, \phi\to\psi\in\Phi$ for all formulae $\phi,\psi\in\Phi$,
    \item $[C]_p\phi\in\Phi$ for each coalition $C$, each real number $p$ such that $0\le p\le 1$, and each formula $\phi\in \Phi$.
\end{enumerate}
\end{definition}
In other words, $\Phi$ is the language specified by the following grammar
$$
\phi := v \;|\; \neg\phi\;|\;\phi\to\phi\;|\;[C]_p\phi.$$ 
We assume that Boolean constants $\top$ and $\bot$ are defined in our languages in the standard way.  By $\Phi^+$ we denote the subset of $\Phi$ that contains all formulae in $\Phi$ that do not use empty coalitions. In other words, language $\Phi^+$ could be defined as in Definition~\ref{Phi} but with an additional assumption that coalition $C$ is not empty.

Let $X^Y$ be the set of all functions from set $Y$ to set $X$. 
\begin{definition}\label{transition system}
A tuple $(S,F,D,P,\pi)$ is a stochastic game, if
\begin{enumerate}
\item $S$ is a set (of states),
\item $F \subseteq S$ is a set (of failure states),
\item $D$ is a nonempty set (domain of actions),
\item $P$ is a function from set  $S\times D^A\times S$ into set $[0,1]$ such that $$\sum_{s'\in S}P(s,\delta,s') = 1$$ for each state $s\in S$ and each function $\delta\in D^A$,
\item $\pi$ is a function from propositional variables into subsets of $S$.
\end{enumerate}
\end{definition}
By $\cF$ we denote the complement $S\setminus F$ of the set $F$. A function from set $D^A$ is called a {\em complete action profile}. 

In the introductory example depicted in Figure~\ref{intro-example game figure}, the set of agents $A$ consists of car $a$ and car $b$. The set of states $S$ is $\{ab, f_c, p, f_t, ba\}$ and the set of failure states $F$ is $\{f_c, f_t\}$. The domain of actions $D$ is $\{-,0,+\}$. Although formally a complete action profile $\delta$ is a function from set of all agents $A$ to the domain of actions $D$, in the case of our introductory example it is more convenient to refer to such profiles by pairs $\alpha\beta$, where $\alpha=\delta(a)$ and $\beta=\delta(b)$. The function $P$ is specified by labels on the directed edges in the diagram. We use commas to denote multiple functions with the same probability. For example, the label  ``$-0,0+ / 0.9$'' on the directed edge from state $p$ to state $ab$ means that $P(p,-0,ab) = 0.9$ and $P(p, 0+, ab) = 0.9$.


Next is the key definition of this article. Its item 4 formally specifies the semantics of the modality $[C]_p$. In this definition we use term {\em action profile of a coalition} to refer to a function $\delta$ that assigns an action $\delta(a)$ to each agent $a$ of a coalition $C$. Also, note that for any two relations $R_1,R_2\subseteq X\times Y$, we have $R_1\subseteq R_2$ if every pair $(x,y)\in X\times Y$ in relation $R_1$ is also in relation $R_2$. If $f$ and $g$ are partial functions (functional relations), then $f\subseteq g$ means that function $g$ is an extension of function $f$.

\begin{definition}\label{sat}
For any state $s\in \cF$ of a stochastic game $(S,F,D,P,\pi)$ and any formula $\phi\in\Phi$, the satisfiability relation $s\Vdash\phi$ is defined recursively as follows:
\begin{enumerate}
    \item $s\Vdash v$ if $s\in \pi(v)$, for any propositional variable $v$,
    \item $s\Vdash \neg\phi$ if $s\nVdash\phi$,
    \item $s\Vdash \phi\to\psi$ if $s\nVdash\phi$ or $s\Vdash\psi$,
    \item $s\Vdash [C]_p\phi$ when
    there is an action profile $\delta\in D^C$ of coalition $C$ such that for any complete action profile $\delta'\in D^A$ if $\delta\subseteq \delta'$, then 
    \begin{enumerate}
        \item $\sum_{t\in \cF}P(s,\delta',t)\ge p$,
        \item if $P(s,\delta',t)>0$, then $t\Vdash \phi$, for each $t\in \cF$.
    \end{enumerate}
    

\end{enumerate}
\end{definition}

\section{Logical Systems}\label{logical systems section}

In this section we introduce the axioms and the inference rules of logical systems $\mathcal{L}$ and $\mathcal{L}^+$ in languages $\Phi$ and $\Phi^+$ respectively. In addition to propositional tautologies in the corresponding language, each system contains the following axioms:
\begin{enumerate}
\item Cooperation: $[C_1]_{p}(\phi\to\psi)\to([C_2]_{q}\phi\to[C_1\cup C_2]_{\max\{p,q\}}\psi)$,\\ where $C_1\cap C_2=\varnothing$,
\item Monotonicity: $[C]_{p}\phi\to [C]_q\phi$, where $q\le p$,
\item Unachievability of Falsehood: $\neg [C]_p\bot$, where $p>0$.
\end{enumerate}
The Cooperation axiom in the form without subscripts goes back to Marc Pauly~\cite{p01illc,p02}. Informally, it says that two coalitions can combine their strategies to achieve a common goal. The assumption that coalitions $C_1$ and $C_2$ are disjoint is important because a hypothetical common agent of these two coalitions might be required to choose different actions under strategies of these two coalitions.

Our version of the Cooperation axiom adds probability of non-failure subscript to the original version of this axiom. Perhaps one might think that the conclusion of the axiom should have subscript $\min\{p,q\}$ rather than $\max\{p,q\}$. This is not true because, according to Definition~\ref{sat}, statement $[C]_p\phi$ means that coalition $C$ has a strategy to achieve $\phi$ with probability of non-failure of at least $p$ {\em regardless} of what actions are chosen by the other agents. 

The Monotonicity axiom says that if a coalition $C$ can achieve goal $\phi$ with probability of non-failure of at least $p$, then coalition $C$ can achieve $\phi$ with probability of non-failure of at least $q$, where $q \le p$.

Finally, the Unachievability of Falsehood axiom says that no coalition can achieve falsehood with a positive probability. 

We write $\vdash_{\mathcal{L}}\phi$ if formula $\phi\in \Phi$ is provable from the above axioms using the Modus Ponens and the Necessitation inference rules:
$$
\dfrac{\phi,\phi\to\psi}{\psi}
\hspace{10mm}
\dfrac{\phi}{[C]_0\phi}
.
$$
Notice that the Necessitation inference rule with positive subscript is not, generally speaking, valid. Indeed, formula $\top$ is universally true but coalition $C$ may not have a strategy that guarantees the non-failure of the system with a positive probability. Thus, $[C]_p\top$ is not a universally true formula for $p > 0$. 

Let $\vdash_{\mathcal{L}^+}\phi$ if formula $\phi\in \Phi^+$ is provable (using only formulae in language $\Phi^+)$ from the above axioms using the Modus Ponens, the Necessitation, and the Monotonicity
$$
\dfrac{\phi\to\psi}{[C]_p\phi\to [C]_p\psi}
$$
inference rules. We excluded the Monotonicity rule from system $\mathcal{L}$ because, as we show below, it is derivable in $\mathcal{L}$.

\begin{lemma}
Monotonicity inference rule is derivable in system $\mathcal{L}$.
\end{lemma}
\begin{proof}
Suppose that $\vdash_{\mathcal{L}}\phi\to\psi$. Thus, $\vdash_{\mathcal{L}}[\varnothing]_0(\phi\to\psi)$ by the Necessitation inference rule. Consider now the following instance of the Cooperation axiom:
$
[\varnothing]_0(\phi\to\psi)\to ([C]_p\phi\to [C]_p\psi)
$.
Therefore, $\vdash_{\mathcal{L}}[C]_p\phi\to [C]_p\psi$ by the Modus Ponens inference rule.

\end{proof}

We write $X\vdash_{\mathcal{L}}\phi$ (or $X\vdash_{\mathcal{L}^+}\phi$) if formula $\phi\in\Phi$ (or $\phi\in\Phi^+$)   is provable from the theorems of logical system $\mathcal{L}$ (or $\mathcal{L}^+$) and a set of additional axioms $X$ using only the Modus Ponens inference rule. Note that if set $X$ is empty, then statement $X\vdash_{\mathcal{L}}\phi$ is equivalent to $\vdash_{\mathcal{L}}\phi$ and statement $X\vdash_{\mathcal{L}^+}\phi$ is equivalent to $\vdash_{\mathcal{L}^+}\phi$. We often write $\vdash\phi$ and $X\vdash\phi$ if it is clear from the context which logical system we refer to. We say that set $X$ is consistent if $X\nvdash\bot$. 

\begin{lemma}[deduction]\label{deduction lemma}
For either $\mathcal{L}$ or $\mathcal{L}^+$, if $X,\phi\vdash\psi$, then $X\vdash\phi\to\psi$.
\end{lemma}
\begin{proof}
Suppose that sequence $\psi_1,\dots,\psi_n$ is a proof from set $X\cup\{\phi\}$ and the theorems of our logical system that uses the Modus Ponens inference rule only. In other words, for each $k\le n$, either
\begin{enumerate}
    \item $\vdash\psi_k$, or
    \item $\psi_k\in X$, or
    \item $\psi_k$ is equal to $\phi$, or
    \item there are $i,j<k$ such that formula $\psi_j$ is equal to $\psi_i\to\psi_k$.
\end{enumerate}
It suffices to show that $X,\phi\vdash\psi_k$ for each $k\le n$. We prove this by induction on $k$ through considering the four cases above separately.

\vspace{1mm}
\noindent{\bf Case 1}: $\vdash\psi_k$. Note that $\psi_k\to(\phi\to\psi_k)$ is a propositional tautology, and thus, is an axiom of our logical system. Hence, $\vdash\phi\to\psi_k$ by the Modus Ponens inference rule. Therefore, $X\vdash\phi\to\psi_k$. 

\vspace{1mm}
\noindent{\bf Case 2}: $\psi_k\in X$. Then, $X\vdash\psi_k$.

\vspace{1mm}
\noindent{\bf Case 3}: formula $\psi_k$ is equal to $\phi$. Thus, $\phi\to\psi_k$ is a propositional tautology. Therefore, $X\vdash\phi\to\psi_k$. 

\vspace{1mm}
\noindent{\bf Case 4}:  formula $\psi_j$ is equal to $\psi_i\to\psi_k$ for some $i,j<k$. Thus, by the induction hypothesis, $X\vdash\phi\to\psi_i$ and $X\vdash\phi\to(\psi_i\to\psi_k)$. Note that formula 
$
(\phi\to\psi_i)\to((\phi\to(\psi_i\to\psi_k))\to(\phi\to\psi_k))
$
is a propositional tautology. Therefore, $X\vdash \phi\to\psi_k$ by applying the Modus Ponens inference rule twice.
\end{proof}

Note that it is important for the above proof that $X\vdash\phi$ stands for derivability only using the Modus Ponens inference rule. For example, if the Necessitation inference rule is allowed, then the proof will have to include one more case where $\psi_k$ is formula $[C]_0\psi_i$ for some coalition $C\subseteq A$, and some integer $i< k$. In this case we will need to prove that if $X\vdash \phi\to\psi_i$, then $X\vdash \phi\to[C]_0\psi_i$, which is not true.

\begin{lemma}[Lindenbaum]\label{Lindenbaum's lemma}
For either $\mathcal{L}$ or $\mathcal{L}^+$, any consistent set of formulae can be extended to a maximal consistent set of formulae.
\end{lemma}
\begin{proof}
The standard proof of Lindenbaum's lemma applies here~\cite[Proposition 2.14]{m09}. However, since the formulae in our logical systems use real numbers in subscript, the set of formulae  is uncountable. Thus, the proof of Lindenbaum's lemma in our case relies on the Axiom of Choice.
\end{proof}

We conclude this section by giving an example of a formal derivation in our logical systems. This result is used later in the proof of the completeness. 

\begin{lemma}\label{set monotonicity}
For any coalitions $C,D\subseteq A$, if $C\subseteq D$, then
\begin{enumerate}
    \item $\vdash_{\mathcal{L}} [C]_p\phi \to [D]_p\phi$ for each formula $\phi\in\Phi$,
    \item $\vdash_{\mathcal{L}^+} [C]_p\phi \to [D]_p\phi$ for each formula $\phi\in\Phi^+$, where set $C$ is not empty.
\end{enumerate}
\end{lemma}
\begin{proof}
We give a common proof for both parts of the lemma.
If $C=D$, then $\vdash [C]_p\phi \to [D]_p\phi$ because formula $[C]_p\phi \to [D]_p\phi$ is a propositional tautology. 

Suppose now that $C\subsetneq D$. Thus set $D\setminus C$ is not empty.
Note that $\phi\to\phi$ is a propositional tautology. Thus, $\vdash[D\setminus C]_0(\phi\to\phi)$ by the Necessitation inference rule. At the same time, because $(D\setminus C) \cap C = \varnothing$, the following formula is an instance of the Cooperation axiom:
$$
[D\setminus C]_0(\phi\to\phi)\to([C]_p\phi\to[(D\setminus C)\cup C]_{\max\{0,p\}}\phi).
$$
Hence, by the Modus Ponens inference rule,
$$
\vdash [C]_p\phi\to[(D\setminus C)\cup C]_{\max\{0,p\}}\phi.
$$
Then, $\vdash [C]_p\phi\to[D]_p\phi$, because $C\subseteq D$ and $0\le p$.
\end{proof}

\section{Soundness}\label{soundness section}

In this section we prove the soundness of each of our axioms as a separate lemma. The same proof applies to both system $\mathcal{L}$ and system $\mathcal{L}^+$. The soundness of the systems is stated in the end of the section as Theorem~\ref{soundness theorem}. 

\begin{lemma}\label{cooperation sound}
For any state $s\in \cF$ of a stochastic game $(S,F,D,P,\pi)$, any coalitions $C_1$ and $C_2$, any formulae $\phi,\psi\in\Phi$, and any real numbers $p,q$ such that $0\le p,q\le 1$,  if $s \Vdash [C_1]_{p}(\phi\to\psi)$, $s \Vdash [C_2]_{q}\phi$, and  $C_1\cap C_2=\varnothing$, then $s \Vdash [C_1\cup C_2]_{\max\{p,q\}}\psi$.
\end{lemma}
\begin{proof}
By Definition~\ref{sat}, assumption $s \Vdash [C_1]_{p}(\phi\to\psi)$ implies that there is an action profile $\delta_1 \in D^{C_1}$ such that for any complete action profile $\delta'\in D^A$, if $\delta_1\subseteq \delta'$, then
    \begin{enumerate}
        \item $\sum_{t\in \cF}P(s,\delta',t)\ge p$,
        \item if $P(s,\delta',t)>0$, then $t\Vdash \phi \to \psi$, for each $t\in \cF$.
    \end{enumerate}
Additionally, by Definition~\ref{sat}, assumption $s \Vdash [C_2]_{q}\phi$ implies that there is an action profile $\delta_1 \in D^{C_2}$ such that for any complete action profile $\delta'\in D^A$ if $\delta_2\subseteq \delta'$, then
    \begin{enumerate}
        \item[3.] $\sum_{t\in \cF}P(s,\delta',t)\ge q$,
        \item[4.] if $P(s,\delta',t)>0$, then $t\Vdash \phi$, for each $t\in \cF$.
    \end{enumerate}
Let action profile $\delta$ of coalitions $C_1 \cup C_2$ be defined as 
\begin{equation}\label{delta definition}
    \delta(a) = 
    \begin{cases}
        \delta_1(a), & \mbox{if } a\in C_1,\\
        \delta_2(a), & \mbox{if } a\in C_2.
    \end{cases}
\end{equation}
Action profile $\delta$ is well-defined because coalitions $C_1$ and $C_2$ are disjoint by an assumption of the lemma. 

Consider an arbitrary complete action profile $\delta'$ such that $\delta\subseteq\delta'$. Note that
\begin{eqnarray}
    && \delta_1 \subseteq \delta \subseteq \delta',\label{delta1 subset}\\
    && \delta_2 \subseteq \delta \subseteq \delta'\label{delta2 subset}
\end{eqnarray}
by equation~(\ref{delta definition}) and the assumption $\delta\subseteq\delta'$.
Thus, by Definition~\ref{sat} and the above assumptions 1, 2, 3, and 4,
\begin{enumerate}
        \item $\sum_{t\in \cF}P(s,\delta',t)\ge max\{p,q\}$,
        \item if $P(s,\delta',t)>0$, then $t\Vdash \psi$, for each $t\in \cF$.
\end{enumerate}
Therefore, $s\Vdash[C_1\cup C_2]_{\max{p,q}}\psi$ by Definition~\ref{sat}.
\end{proof}

\begin{lemma}\label{monotonicity sound}
For any state $s\in \cF$ of a stochastic game $(S,F,D,P,\pi)$, any coalition $C$, any formula $\phi\in\Phi$, and any real numbers $p,q$ such that $0\le q\le p\le 1$, if $s \Vdash [C]_p\phi$, then $s \Vdash [C]_q\phi$.
\end{lemma}
\begin{proof}
By Definition~\ref{sat}, assumption $s \Vdash [C]_p\phi$ implies that there is an action profile $\delta_1 \in D^{C}$ such that for any complete action profile $\delta'\in D^A$ if $\delta_1\subseteq \delta'$, then
    \begin{enumerate}
        \item $\sum_{t\in \cF}P(s,\delta',t)\ge p$,
        \item if $P(s,\delta',t)>0$, then $t\Vdash \phi$, for each $t\in \cF$.
    \end{enumerate}
Note that $\sum_{t\in \cF}P(s,\delta',t)\ge p \ge q$ by assumption $q\le p$ of the lemma. Therefore, $s \Vdash [C]_q\phi$ by Definition~\ref{sat}.
\end{proof}

\begin{lemma}\label{falsehood sound}
For any state $s\in \cF$ of a stochastic game $(S,F,D,P,\pi)$, any coalition $C$, and any real number $p$, if $0 < p\le 1$, then $s \nVdash [C]_p\bot$.
\end{lemma}
\begin{proof}
Suppose that $s\Vdash [C]_p\bot$. Thus, by Definition~\ref{sat}, there is an action profile $\delta_1 \in D^{C}$ such that for any complete action profile $\delta'\in D^A$ if $\delta_1\subseteq \delta'$, then
    \begin{enumerate}
        \item $\sum_{t\in \cF}P(s,\delta',t)\ge p$,
        \item if $P(s,\delta',t)>0$, then $t\Vdash \phi$, for each $t\in \cF$.
    \end{enumerate}
Notice that $\sum_{t\in \cF}P(s,\delta',t)\ge p > 0$ due to the assumption $0 < p$ of the lemma. Hence, there exists state $t\in \cF$ such that $P(s,\delta', t) > 0$. Thus, $t \Vdash \bot$ by item 2 above, which contradicts the definition of $\bot$ and Definition~\ref{sat}. 
\end{proof}

The soundness theorem for our logical systems with respect to the  semantics described above follows from Lemma~\ref{cooperation sound}, Lemma~\ref{monotonicity sound}, and Lemma~\ref{falsehood sound}.  
\begin{theorem}\label{soundness theorem}
For either $\mathcal{L}$ or $\mathcal{L}^+$, if $\,\vdash \phi$, then $s \Vdash \phi$ for each state $s\in \cF$ of each stochastic game $(S,F,D,P,\pi)$. 
\end{theorem}

\section{Completeness}\label{completeness section}

In this section we prove weak completeness of system $\mathcal{L}$ and strong completeness of $\mathcal{L}^+$ with respect to the semantics of stochastic games. These results are stated later in this section as Theorem~\ref{weak completeness L theorem} and Theorem~\ref{strong completeness L+ theorem}.

Let $\Psi$ be either language $\Phi$ or $\Phi^+$ and $\Sigma$ be any subset of $\Psi$ such that (a) $\Sigma$ is closed with respect to subformulae and (b) if $\sigma\in\Sigma$, then $\neg\sigma\in\Sigma$, unless the formula $\sigma$ itself is a negation. We distinguish $\Sigma$ from the whole set $\Psi$ so that later set $\Sigma$ could be assumed to be finite. We start the proof by defining the canonical stochastic game $G(\Psi,\Sigma)=(S,F,D,P,\pi)$.

\begin{definition}\label{canonical state}
Set $S$ consists of all maximal consistent subsets of $\Sigma$ and an additional ``failure'' state $f$.
\end{definition}

\begin{definition}\label{canonical failure}
$F=\{f\}$.
\end{definition}

\begin{definition}
$D$ is the set of all pairs $(\phi,p)$ where $\phi\in\Sigma$ and $p$ is an arbitrary real number.
\end{definition}

Informally, by choosing the action $(\phi,p)$, the agent is requesting the game to transition to a non-failure state with probability at least $p$ and formula $\phi$ to be true at that state. The game might grant or ignore this request. In particular, the game ignores the request if $p\notin [0,1]$.

Next, we define function $P$. This is done in Definition~\ref{canonical P} through auxiliary functions  $\mu(s,\delta)$ and $T(s,\delta)$. Function  $\mu(s,\delta)$ specifies the probability of the canonical game to transition  from state $s$ under complete action profile $\delta$ into a into non-failure state. For each $[C]_p\phi\in s$ we want the game to transition to a non-failure state with probability at least $p$ if all members of coalition $C$ choose action $(\phi,p)$.  Thus, we define $\mu(s,\delta)$ to be the maximum among such $p$. In the definition below we assume that the maximum of the empty set is equal to 0.

\begin{definition}\label{canonical mu}
For each state $s\in\cF$ and each complete action profile $\delta\in D^A$, let
$\mu(s,\delta)=\max\{p\;|\; [C]_p\phi\in s, \forall a\in C (\delta(a)=(\phi,p))\}.$
\end{definition}

\begin{lemma}\label{P limits L}
If $\Psi=\Phi$ and set $\Sigma$ is finite, then for each state $s\in \cF$ and each profile $\delta\in D^A$, value $\mu(s,\delta)$ is well-defined and $\mu(s,\delta)\in [0,1]$.
\end{lemma}
\begin{proof}
Consider set $X = \{p\;|\; [C]_p\phi\in s, \forall a\in C (\delta(a)=(\phi,p))\}$.
Note that $X \subseteq [0,1]$ by Definition~\ref{Phi}. To prove that value $\mu(s,\delta)$ is well-defined by Definition~\ref{canonical mu}, it suffices to show that set $X$ is finite. Indeed, set $s$ is finite because it is a subset of finite set $\Sigma$. Therefore, set $X$ is finite by the choice of set $X$. 
\end{proof}

\begin{lemma}\label{P limits L+}
If $\Psi=\Phi^+$, then for each state $s\in \cF$ and each profile $\delta\in D^A$, value $\mu(s,\delta)$ is well-defined and $\mu(s,\delta)\in [0,1]$.
\end{lemma}
\begin{proof}
Consider set $X = \{p\;|\; [C]_p\phi\in s, \forall a\in C (\delta(a)=(\phi,p))\}$.
Note that $X \subseteq [0,1]$ by Definition~\ref{Phi}. To prove that value $\mu(s,\delta)$ is well-defined by Definition~\ref{canonical mu}, it suffices to show that set $X$ is finite. Recall that set of all agents $A$ is finite. Thus, set $\{p\;|\; \exists a\in A\,\exists \phi\in \Phi\,(\delta(a)=(\phi,p))\}$ is finite. Therefore, set $X$ is finite because any coalition $C$ in a formula $[C]_p\phi\in \Phi^+$ is nonempty.
\end{proof}


Function $T(s,\delta)$ specifies all non-failure states to which the game is able to transition from state $s$ under complete action profile $\delta$ with {\em non-zero probability}. Informally, if $[C]_p\phi\in s$ and all members of coalition $C$ choose action $(\phi,p)$, then statement $\phi$ belongs to each set in $T(s,\delta)$.


\begin{definition}\label{canonical T}
For each state $s\in\cF$ and each complete action profile $\delta\in D^A$, let $T(s,\delta)$ be the set of all $s'\in \cF$ such that
$$
\{\phi\;|\; [C]_p\phi\in s, \forall a\in C (\delta(a)=(\phi,p))\}\subseteq s'.
$$
\end{definition}

We are now ready to define function $P(s, \delta, s')$ that specifies the probability of the canonical game to transition from a state $s$ to a state $s'$ under a complete action profile $\delta$.
\begin{definition}\label{canonical P}
For each state $s \in S$, each complete action profile $\delta \in D^A$, and each state $s' \in S$, 
$$P(s, \delta, s') =
\begin{cases}
\dfrac{\mu(s, \delta)}{|T(s, \delta)|} , & \mbox{ if $s\in \cF$ and $s'\in T(s, \delta)$},\\
1-\mu(s,\delta),& \mbox{ if $s\in \cF$ and $s'=f$},\\
1, & \mbox{ if $s=s'=f$},\\
0, & \mbox{ otherwise},
\end{cases}
$$
where $|T(s, \delta)|$ is the size of set $T(s, \delta)$.
\end{definition}

We prove that $\sum_{s' \in \cF}P(s, \delta, s') = 1$ in Lemma~\ref{sum is one}. But first we show that $\mu(s,\delta)$ is an upper bound on the sum of probabilities of transitioning to a non-failure state.

\begin{lemma}\label{P le mu} 
For each state $s \in \cF$, each complete action profile $\delta \in D^A$,
$$\sum_{s' \in \cF}P(s, \delta, s') \le \mu(s, \delta).$$
\end{lemma}
\begin{proof} We consider the following two cases separately:

\noindent{\bf Case I:} $T(s, \delta) = \varnothing$. Then, by Definition~\ref{canonical P} and either Lemma~\ref{P limits L} or Lemma~\ref{P limits L+},

\begin{eqnarray*}
\sum_{s'\in \cF}P(s,\delta,s') &=& \sum_{s' \in T(s, \delta)}P(s,\delta,s') + \sum_{s' \in \cF \setminus T(s, \delta)}P(s,\delta,s')\\ 
&=&  \sum_{s' \in \varnothing}P(s,\delta,s') + \sum_{s' \in \cF \setminus T(s, \delta)}0 = 0 \le \mu(s, \delta).
\end{eqnarray*}

\noindent{\bf Case II:} $T(s, \delta) \neq \varnothing$. Then, by Definition~\ref{canonical P}, 

\begin{eqnarray*}
\sum_{s'\in \cF}P(s,\delta,s') &=& \sum_{s' \in T(s, \delta)}P(s,\delta,s') + \sum_{s' \in \cF \setminus T(s, \delta)}P(s,\delta,s')\\ 
&=&  \sum_{s' \in T(s, \delta)}\dfrac{\mu(s, \delta)}{|T(s, \delta)|} + \sum_{s' \in \cF \setminus T(s, \delta)}0\\
&=& \mu(s, \delta) + 0 \le \mu(s, \delta).
\end{eqnarray*}

\end{proof}

\begin{definition}\label{canonical pi}
$\pi(v)=\{s\in \cF\;|\; v\in s\}$.
\end{definition}
This concludes the definition of the canonical stochastic game $G(\Psi,\Sigma)=(S,F,D,P,\pi)$ in cases when either $\Sigma$ is finite or $\Psi=\Phi^+$. Throughout the rest of this section we assume that one of these two conditions is true.

The next lemma is the key lemma in the proof of the completeness. It shows that if $\neg[C]_p\phi\in s$, then in state $s$ coalition $C$ has no strategy to transition to a non-failure state with probability at least $p$ and to guarantee that $\phi$ is true in that state.

\begin{lemma}\label{s' exists}
For each state $s\in \cF$, each formula $\neg[C]_p\phi\in s$, and each $\delta\in D^C$, there is $\delta'\in D^A$ such that $\delta\subseteq\delta'$ and one of the following is true
\begin{enumerate}
    \item $\mu(s, \delta') < p$ or 
    \item there is a state $s' \in \cF$ where $P(s,\delta',s')>0$ and $\neg\phi \in s'$.
\end{enumerate}
\end{lemma}

\begin{proof}
Consider function $\delta'\in D^A$ such that
\begin{equation}\label{choice of delta'}
\delta'(a)=
\begin{cases}
\delta(a), & \mbox{ if } a\in C,\\
(\top,-1), & \mbox{ otherwise}.
\end{cases}
\end{equation}
Suppose that $\mu(s, \delta') \ge p$. We will show that there is a state $s' \in \cF$ such that $P(s,\delta',s')>0$ and $\neg\phi \in s'$.
Consider set
$$X_0=\{\neg\phi\}\cup\{\psi\;|\; [B]_q\psi\in s, \forall a\in B(\delta'(a)=(\psi,q))\}.
$$
First, we prove that set $X_0$ is consistent. Suppose the opposite, thus there must exist formulae $[B_1]_{q_1}\psi_1,\dots, [B_n]_{q_n}\psi_n\in s$ such that
\begin{eqnarray}
&&\forall i\le n\; \forall a\in B_i\; (\delta'(a)=(\psi_i,q_i))\label{choice of votes}\\
&&\psi_1,\dots,\psi_n\vdash \phi.\label{provable}
\end{eqnarray}

Without loss of generality, we can assume that formulae $\psi_1,\dots,\psi_n$ are distinct. Note that sets $B_1,\dots,B_n$ are pairwise disjoint because of statement~(\ref{choice of votes}).  Due to Definition~\ref{canonical P}, 
 \begin{equation}\label{qs are small}
     q_1,\dots,q_n\le \mu(s,\delta').
 \end{equation}
 Additionally, by Definition~\ref{canonical P} and the assumption of the case, we can suppose that there is an integer $m$ such that $1\le m\le n$ and
 \begin{equation}\label{choice of qm}
     q_m=\mu(s,\delta').
 \end{equation} 
Furthermore, we can assume that there is $n'\le n$ such that $B_i\subseteq C$ for each $i\le n'$ and $B_i\nsubseteq C$ for each $i>n'$.

Let us first show that $m\le n'$. Indeed, suppose that there is $a_0\in B_m\setminus C$. Thus, $\delta'(a_0)=(\top,-1)$ by equation~(\ref{choice of delta'}). Hence, $q_m=-1$ due to equation~(\ref{choice of votes}). Recall that  $\mu(s,\delta')=q_m$ by the choice of index $m$. Thus $\mu(s,\delta')=-1$, which contradicts Lemma~\ref{P limits L} (or Lemma~\ref{P limits L+} in case of system $\mathcal{L}^+$). Therefore, $m\le n'$.

Next, note that for each $i> n'$ we have  $\psi_i=\top$  because $B_i\nsubseteq C$ and  due to equality~(\ref{choice of delta'}) and equality~(\ref{choice of votes}). Hence, $\psi_1,\dots,\psi_{n'}\vdash \phi$ by statement~(\ref{provable}).
By Lemma~\ref{deduction lemma} applied $n$ times, 
$$
\vdash \psi_1\to(\psi_2\to\dots (\psi_{n'}\to \phi)\dots).
$$ 
Note that $n'\neq 0$ because $1\le m\le n'$.
So, by the Monotonicity inference rule,
$$
\vdash[B_1]_{q_1}\psi_1 \to [B_1]_{q_1}(\psi_2\to\dots (\psi_{n'}\to \phi)\dots)).
$$
By the Modus Ponens inference rule,
$$
[B_1]_{q_1}\psi_1 \vdash [ B_1]_{q_1}(\psi_2\to\dots (\psi_{n'}\to \phi)\dots)).
$$
By the Cooperation axiom and the Modus Ponens rule,
\begin{eqnarray*}
&&[B_1]_{q_1}\psi_1, [B_2]_{q_2}\psi_2 \vdash[B_1\cup B_2]_{\max\{q_1,q_2\}}(\psi_3\to\dots (\psi_{n'}\to \phi)\dots)).
\end{eqnarray*}
By repeating the previous step $n-2$ more times,
$$[B_1]_{q_1}\psi_1,\dots,[B_{n'}]_{q_{n'}}\psi_{n'} \vdash
[B_1\cup\dots \cup B_{n'}]_{\max\{q_1,\dots,q_{n'}\}}\phi.$$
Thus, by the choice of formulae $[B_1]_{q_1}\psi_1,\dots, [B_{n'}]_{q_{n'}}\psi_{n'}$,
$$s \vdash
[B_1\cup\dots \cup B_{n'}]_{\max\{q_1,\dots,q_{n'}\}}\phi.$$
Then, by Lemma~\ref{set monotonicity} and because $B_1,\dots,B_{n'}\subseteq C$,
$$s \vdash
[C]_{\max\{q_1,\dots,q_{n'}\}}\phi.$$
Recall that $q'\le n'$. Thus, $\max\{q_1,\dots,q_{n'}\}=\mu(s,\delta')$ by inequality~(\ref{qs are small}) and equation~(\ref{choice of qm}). Hence,  
$s \vdash[C]_{\mu(s,\delta')}\phi$. Thus, $s \vdash
[C]_{p}\phi$ by the Monotonicity axiom and the assumption $\mu(s, \delta') \ge p$. Then, $\neg[C]_{p}\phi\notin s$ due to consistency of set $s$, which contradicts the assumption of the lemma. Therefore, set $X_0$ is consistent. By Lemma~\ref{Lindenbaum's lemma},  there is a maximal consistent extension $s'$ of set $X_0$. Note that $\neg\phi\in s'$ by the choice of set $X_0$.

Note that $s'\in T(s,\delta')$ by Definition~\ref{canonical T} and the choice of sets $X_0$ and $s'$. Thus, set $T(s,\delta')$ is not empty. Hence, by the assumption of the case,
$$
P(s,\delta',s')=\dfrac{\mu(s,\delta')}{|T(s,\delta')|}>0.
$$
This concludes the proof of the lemma.
\end{proof}

Recall that we left unproven the fact that $\sum_{s' \in \cF}P(s, \delta, s') = 1$. This will be shown in Lemma~\ref{sum is one} using the following auxiliary lemma.

\begin{lemma}\label{mu = 0}
For each state $s \in \cF$ and each complete action profile $\delta \in D^A$, if set $T(s,\delta)$ is empty, then $\mu(s, \delta) = 0$.
\end{lemma}
\begin{proof}
Suppose that $\mu(s, \delta) \neq 0$. Thus, $\mu(s, \delta) > 0$ by either Lemma~\ref{P limits L} or Lemma~\ref{P limits L+}. Then, $\neg[A]_{\mu(s, \delta)/2}\bot\in s$ by the Unachievability of Falsehood axiom. Hence, by Lemma~\ref{s' exists} there is a complete action profile $\delta'\in D^A$ such that $\delta\subseteq\delta'$ and one of the following is true
\begin{enumerate}
    \item $\mu(s, \delta') < \mu(s, \delta)/2$ or 
    \item there is a state $s' \in \cF$ where $P(s,\delta',s')>0$ and $\neg\bot \in s'$.
\end{enumerate}
Note that assumption $\delta\subseteq\delta'$ implies that $\delta=\delta'$ because $\delta$ is a complete action profile. Thus, $\mu(s, \delta') = \mu(s, \delta) > \mu(s, \delta)/2$ by either Lemma~\ref{P limits L} or Lemma~\ref{P limits L+}. Hence, there is a state $s' \in \cF$ such that $P(s,\delta,s')=P(s,\delta',s')>0$. Then, $s'\in T(s,\delta)$ by Definition~\ref{canonical P}. Therefore, set $T(s,\delta)$ is not empty.

\end{proof}

\begin{lemma}\label{sum is one}
For each state $s \in S$ and each complete action profile $\delta \in D^A$, $$\sum_{s'\in S}P(s,\delta,s') = 1.$$
\end{lemma}
\begin{proof} We consider the following three cases separately.

\noindent{\bf Case I}: $s = f$. Thus, by Definition~\ref{canonical P}, 
$$\sum_{s'\in S}P(s,\delta,s') = \sum_{s'\in S}P(f,\delta,s')= P(f,\delta,f)+ \sum_{s'\in \cF}P(f,\delta,s') = 1 + \sum_{s'\in \cF}0 = 1.$$

\noindent{\bf Case II}: $s \in \cF$ and $T(s, \delta) = \varnothing$. Hence, $\mu(s, \delta) = 0$ by Lemma~\ref{mu = 0}. Then, by Definition~\ref{canonical P}, 
\begin{eqnarray*}
\sum_{s'\in S}P(s,\delta,s') &=& P(s,\delta,f) + \sum_{s' \in T(s, \delta)}P(s,\delta,s') + \sum_{s' \in \cF \setminus T(s, \delta)}P(s,\delta,s')\\ 
&=&  1 - \mu(s, \delta) + \sum_{s' \in \varnothing}P(s,\delta,s') + \sum_{s' \in \cF \setminus T(s, \delta)}0 = 1.
\end{eqnarray*}

\noindent{\bf Case III}: $s \in \cF$ and $T(s, \delta) \neq \varnothing$. By Definition~\ref{canonical P}, 
\begin{eqnarray*}
\sum_{s'\in S}P(s,\delta,s') &=& P(s,\delta,f) + \sum_{s' \in T(s, \delta)}P(s,\delta,s') + \sum_{s' \in \cF \setminus T(s, \delta)}P(s,\delta,s')\\ 
&=&  1 - \mu(s, \delta) + \sum_{s' \in T(s, \delta)}\dfrac{\mu(s, \delta)}{|T(s, \delta)|} + \sum_{s' \in \cF \setminus T(s, \delta)}0\\
&=& 1 - \mu(s, \delta) + \mu(s, \delta) + 0 = 1.
\end{eqnarray*}

\end{proof}


The following lemma shows that if $[C]_p\phi\in s$, then in state $s$ coalition $C$ has a strategy which guarantees that the game transitions to a non-failure state with probability at least $p$ and $\phi$ will be true in that state.

\begin{lemma}\label{s' all}
For any state $s\in \cF$ and any formula $[C]_p\phi\in s$, there is an action profile $\delta \in D^C$ such that for any complete action profile $\delta'$ and any state $s'\in \cF$, if $\delta \subseteq \delta'$ and $P(s,\delta', s') > 0$, then $\phi \in s'$.
\end{lemma}
\begin{proof}
Consider any state $s\in \cF$ and any formula $[C]_p\phi$. Let action profile $\delta\in D^C$ be defined as following: $\delta(a)=(\phi,p)$ for each agent $a\in C$. 

Let $s'\in \cF$ be a state and $\delta'\in D^A$ be a complete action profile such that $\delta \subseteq \delta'$ and $P(s,\delta', s') > 0$. Note that $\delta'(a)=\delta(a)=(\phi,p)$ for each agent $a\in C$ by the choice of action profile $\delta$. 

At the same time, $s' \in T(s, \delta')$ because $P(s,\delta', s') > 0$ by Definition~\ref{canonical P}. Therefore, $\phi\in s'$ by Definition~\ref{canonical T} because $[C]_p\phi \in s$ and $\delta'(a)=(\phi,p)$ for each agent $a\in C$.
\end{proof}

The next lemma is the standard induction lemma in the proof of completeness. It brings together the results established in Lemma~\ref{s' exists} and Lemma~\ref{s' all}.

\begin{lemma}\label{main induction}
$\phi\in s$ iff $s\Vdash \phi$ for any formula $\phi\in \Psi$ and any maximal consistent set $s\in \cF$.
\end{lemma}
\begin{proof}
We prove the lemma by structural induction on formula $\phi$. The case when formula $\phi$ is a propositional variable follows from Definition~\ref{canonical pi} and Definition~\ref{sat}. The case when formula $\phi$ is a negation or an implication follows from Definition~\ref{sat} and the maximality and the consistency of set $s$ in the standard way. Let us now suppose that formula $\phi$ has the form $[C]_p\psi$.

\noindent$(\Rightarrow):$ Suppose that $[C]_p\psi\in s$. Thus, by Lemma~\ref{s' all}, there is an action profile $\delta \in D^C$ such that for any complete action profile $\delta'$ and any state $s'\in \cF$, if $\delta \subseteq \delta'$ and $P(s,\delta', s') > 0$, then $\psi \in s'$. Note that statement $\psi \in s'$ is equivalent to $s'\Vdash \psi$ by the induction hypothesis. Thus, for any complete action profile $\delta'$ and any state $s'\in \cF$, if $\delta \subseteq \delta'$ and $P(s,\delta', s') > 0$, then $s'\Vdash \psi$. Therefore, $s\Vdash [C]_p\psi$ by Definition~\ref{sat}.

\noindent$(\Leftarrow):$ Suppose that $s\Vdash [C]_p\psi$. Thus, by Definition~\ref{sat}, there is an action profile $\delta\in D^C$ such that for any complete action profile $\delta'\in D^A$ if $\delta\subseteq \delta'$, then
     \begin{enumerate}
        \item $\sum_{t\in \cF}P(s,\delta',t)\ge p$,
        \item if $P(s,\delta',s')>0$, then $s'\Vdash \psi$, for each $s'\in \cF$.
    \end{enumerate}
Thus, by Lemma~\ref{P le mu}, 
\begin{equation}\label{mu p}
    \mu(s, \delta) \ge p.
\end{equation}

Assume that $[C]_p\psi\notin s$. Thus, $\neg[C]_p\psi\in s$ due to the maximality of set $s$. Hence, by Lemma~\ref{s' exists} there is $\delta'\in D^A$ such that $\delta\subseteq\delta'$ and one of the following is true:
\begin{enumerate}
\setcounter{enumi}{2}
    \item  $\mu(s, \delta') < p$,
    \item there is an $s' \in \cF$ such that $P(s,\delta',s')>0$ and $\neg\phi \in s'$.
\end{enumerate}
Note that statement 3 can not be true due to inequality~(\ref{mu p}). Thus, there is $s'\in \cF$ where $P(s,\delta',s') > 0$ and $\neg\psi\in s'$. Hence, $\psi\notin s$ due to the consistency of set $s'$. Thus, $s'\nVdash\psi$ by the induction hypothesis, which contradicts to statement 2 above.
\end{proof}

We are now ready to state and to prove the (weak) completeness for system $\mathcal{L}$ and strong completeness for system $\mathcal{L}^+$.

\begin{theorem}\label{weak completeness L theorem}
If $\phi\in \Phi$ and $\nvdash_{\mathcal{L}}\phi$, then there is a state $s$ of a stochastic game such that $s\nVdash\phi$.
\end{theorem}
\begin{proof}
Let $\Sigma\subseteq\Phi$ be any {\em finite} set of formulae such that (a) $\neg\phi\in\Sigma$; (b) $\Sigma$ is closed with respect to subformulae; (c) if $\sigma\in\Sigma$, then $\neg\sigma\in\Sigma$, unless formula $\sigma$ itself is a negation. Consider canonical stochastic game $G(\Phi,\Sigma)=(S,F,D,P,\pi)$.

By Lemma~\ref{Lindenbaum's lemma}, assumption $\nvdash_{\mathcal{L}}\phi$ implies that there is a maximal consistent set $s\subseteq\Sigma$ such that $\neg\phi\in s$. Note that $s\in S$ by Definition~\ref{canonical state} and $\phi\notin s$ because set $s$ is consistent. Therefore, $s\nVdash\phi$ by Lemma~\ref{main induction}.

\end{proof}

\begin{theorem}\label{strong completeness L+ theorem}
If $X\subseteq \Phi^+$, $\phi\in\Phi^+$, and $X\nvdash_{\mathcal{L}^+}\phi$, then there is a state $s$ of a stochastic game such that $s\Vdash\chi$ for each $\chi\in X$ and $s\nVdash\phi$.
\end{theorem}
\begin{proof}
Let $(S,F,D,P,\pi)$ be the canonical stochastic game $G(\Phi^+,\Phi^+)$. Suppose that $X \nvdash_{\mathcal{L}^+}\phi$. Hence, set $X\cup \{\neg\phi\}$ is consistent in ${\mathcal{L}^+}$. By Lemma~\ref{Lindenbaum's lemma}, there is a maximal consistent in ${\mathcal{L}^+}$ extension $s\subseteq\Phi^+$ of set $X\cup \{\neg\phi\}$. Then, $s\in S$ by Definition~\ref{canonical state}. Note that $\phi\notin s$ due to the consistency of set $s$. Also, $\chi\in s$ for each $\chi\in X$ because $X\subseteq s$.  Therefore, $s\Vdash\chi$ for each $\chi\in X$ and $s\nVdash\phi$ by Lemma~\ref{main induction}.
\end{proof}

\section{Incompleteness}\label{incompleteness section}

In this section we show that no logical system in language $\Phi$ is strongly complete with respect to the semantics of stochastic games. This result is formally stated as Theorem~\ref{incompleteness theorem}. 

\begin{definition}\label{entails}
$X$ semantically entails $\phi$, written as
$X\vDash\phi$, when for any state $s$ of any stochastic game, if $s\Vdash\chi$ for each $\chi\in X$, then $s\Vdash\phi$.
\end{definition}

\begin{definition}\label{strongly sound}
A logical system $\mathcal{S}$ is {\em strongly sound} with respect to stochastic games if $X\vdash_{\mathcal{S}}\phi$ implies $X\vDash\phi$ for each set of formulae $X$ and each formula $\phi$.
\end{definition}

\begin{definition}\label{strongly complete}
A logical system $\mathcal{S}$ is {\em strongly complete} with respect to stochastic games if $X\vDash\phi$ implies $X\vdash_{\mathcal{S}}\phi$ for each set of formulae $X$ and each formula $\phi$.
\end{definition}

\begin{theorem}[incompleteness]\label{incompleteness theorem}
Any strongly sound with respect to stochastic games logical system in language $\Phi$ is not strongly complete. 
\end{theorem}
\begin{proof}
Suppose that a logical system $\mathcal{S}$ in language $\Phi$ is strongly sound and strongly complete with respect to stochastic games. Consider following infinite subset of $\Phi$:
$$
X=\{[\varnothing]_{1-10^{-n}}\top\;|\; n\ge 0\}=\{[\varnothing]_0\top,[\varnothing]_{0.9}\top, [\varnothing]_{0.99}\top,\dots\}.
$$
\begin{claim}
$X\vDash [\varnothing]_1\top$.
\end{claim}
\begin{proof-of-claim}
Suppose that $s\Vdash [\varnothing]_{1-10^{-n}}\top$ for each $n\ge 0$ for some state $s\in S$ of a stochastic game $(S,F,D,P,\pi)$. Thus, by Definition~\ref{sat}, for each $n\ge 0$ there is an action profile $\delta_n\in D^\varnothing$ such that for each complete action profile $\delta'\in D^A$, if $\delta_n\subseteq\delta'$, then
$$
\sum_{t\in \cF}P(s,\delta',t)\ge 1-10^{-n}.
$$
Since $\delta_n$ is an action profile of an empty coalition, statement $\delta_n\subseteq\delta'$ is vacuously true for any complete action profile $\delta'$. Hence,
for each $n\ge 0$ and each complete action profile $\delta'\in D^A$, 
$$
\sum_{t\in \cF}P(s,\delta',t)\ge 1-10^{-n}.
$$
At the same time,
$$
\sum_{t\in \cF}P(s,\delta',t)\le \sum_{t\in S}P(s,\delta',t) = 1
$$
by Definition~\ref{transition system}. Thus, for each complete action profile $\delta'\in D^A$, 
$$
\sum_{t\in \cF}P(s,\delta',t)=1.
$$
Therefore, 
$s\Vdash [\varnothing]_1\top$ by Definition~\ref{sat}.
\end{proof-of-claim}

The claim above together with the assumption that logical system $\mathcal{S}$ is strongly complete imply that $X\vdash_{\mathcal{S}}[\varnothing]_1\top$ by Definition~\ref{strongly complete}.  Since any derivation can use only finitely many assumptions, there must exist $N\ge 0$ such that
$$
[\varnothing]_0\top,[\varnothing]_{0.9}\top, [\varnothing]_{0.99}\top,\dots,[\varnothing]_{1-10^{-N}}\vdash [\varnothing]_1\top.
$$
Hence, by Definition~\ref{strongly sound} and the assumption that system $\mathcal{S}$ is strongly sound,
\begin{equation}\label{long formula}
[\varnothing]_0\top,[\varnothing]_{0.9}\top, [\varnothing]_{0.99}\top,\dots,[\varnothing]_{1-10^{-N}}\vDash [\varnothing]_1\top.   
\end{equation}

\begin{figure}[ht]
\begin{center}
\scalebox{0.75}{\includegraphics{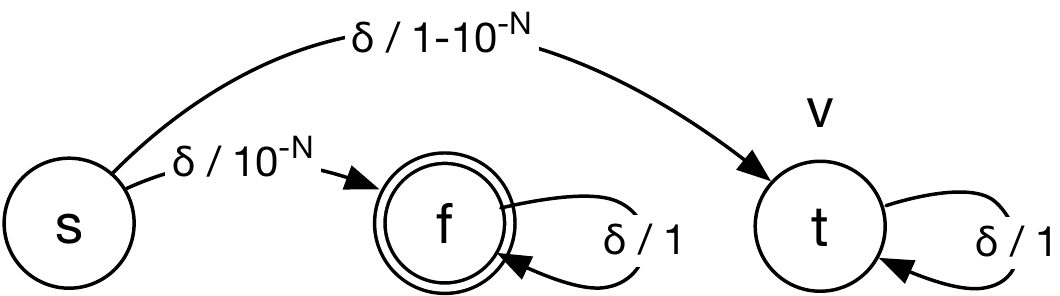}}
\caption{A Stochastic Game.}\label{incompleteness figure}
\end{center}
\end{figure}

Consider now a stochastic game depicted in Figure~\ref{incompleteness figure}. This game has three states: $s$, $f$, and $t$. The domain of actions of the game consists of a single element. Thus, the game has only one complete action profile that we refer to as $\delta$. From state $s$ the game transitions into state $t$ with probability $1-10^{-N}$ and it transitions to failure state $f$ with probability $10^{-N}$. From state $t$ the game always transitions back to state $t$. Additionally, from failure state $f$ the game always transitions back to failure state $f$. Thus,  $s\Vdash[\varnothing]_{1-10^{-n}}\top$ for each $n\le N$ and $s\nVdash [\varnothing]_1\top$ by Definition~\ref{sat}. The last statement, by Definition~\ref{entails}, contradicts statement~(\ref{long formula}).
\end{proof}

\section{Decidability}\label{decidability section}

The languages $\Phi$ and $\Phi^+$ are not countable because we consider modalities labeled by real numbers in the interval $[0,1]$. The languages will become countable if we restrict labels to rational numbers in the same interval. In this case, the set of theorems of logical system $\mathcal{L}$ is decidable because proof of Theorem~\ref{weak completeness L theorem} establishes completeness of our system with respect to the class of {\em finite} stochastic games. Then, system $\mathcal{L}^+$ is also decidable because its theorems are exactly the theorems of $\mathcal{L}$ restricted to language $\Phi^+$.

\section{Conclusion}\label{conclusion section}

In this article we have proposed a notion of a stochastic game with failure states and a logic of coalition power with modality that incorporates probability of non-failure. It has turned out that the logical properties of this modality significantly depend on whether the language allows empty coalitions. We have proven a strong completeness theorem for the language without the empty coalition and a weak completeness theorem for the language with the empty coalition. We have also shown that if the language includes the empty coalition, then no strongly sound and strongly complete axiomatization of the logic exists. Finally, we have observed the decidability of the logical systems discussed in the article.


\end{document}